\documentclass[a4paper,fleqn]{cas-dc}

\usepackage[numbers]{natbib}
\usepackage{amsmath,amssymb,amsthm,mathtools}
\usepackage{booktabs}
\usepackage{multirow}
\usepackage{array}
\usepackage{algorithm}
\usepackage{algpseudocode}
\usepackage{bbm}
\usepackage{enumitem}

\DeclareMathOperator{\rank}{rank}

\newcommand{\R}{\mathbb{R}}
\newcommand{\bbone}{\mathbbm{1}}
\newcommand{\mcal}{\mathcal{M}}
\newcommand{\xt}{\boldsymbol{x}}
\newcommand{\htok}{\boldsymbol{h}}

\newcommand{\Dbar}{\bar D}
\newcommand{\norm}[1]{\left\lVert #1 \right\rVert}

\newcommand{\Dscore}{\mathrm{D}}
\newcommand{\modelname}[1]{\texttt{\hyphenchar\font=`\- #1}}

\newtheorem{definition}{Definition}

\newtheorem{hypothesis}{Hypothesis}
\newtheorem{proposition}{Proposition}

\begin{document}
\let\WriteBookmarks\relax
\def\floatpagepagefraction{1}
\def\textpagefraction{.001}

\shorttitle{D-Score: A Spectral Hidden-State Signal for Hallucination Detection in Large Language Models}
\shortauthors{B. Raimondi et~al.}

\title[mode=title]{D-Score: A Spectral Hidden-State Signal for Hallucination Detection in Large Language Models}

\author[1]{Bianca Raimondi}[orcid=0009-0002-1562-7722]
\cormark[1]
\ead{bianca.raimondi3@unibo.it}
\credit{Conceptualization, Methodology, Software, Writing - original draft}

\author[1]{Davide Evangelista}[orcid=0000-0001-6261-7717]
\ead{davide.evangelista5@unibo.it}
\credit{Conceptualization, Methodology, Writing - original draft}

\author[1]{Maurizio Gabbrielli}[orcid=0000-0003-0609-8662]
\ead{maurizio.gabbrielli@unibo.it}
\credit{Supervision}

\author[1]{Elena {Loli Piccolomini}}[orcid=0000-0002-9951-3564]
\ead{elena.loli@unibo.it}
\credit{Conceptualization, Supervision}

\affiliation[1]{organization={Department of Computer Science and Engineering, University of Bologna},
                city={Bologna},
                country={Italy}}

\cortext[cor1]{Corresponding author}

\begin{abstract}
Large Language Models can produce fluent text that is false, unsupported by the available evidence, or inconsistent with information that appears to be internally represented by the model. We study hallucination detection from the geometry of hidden activations and introduce the \emph{D-Score}, a simple spectral statistic computed from a single forward pass. For a fixed model, layer, and tolerance parameter, the D-Score counts how many singular directions of the hidden activation matrix have singular values that remain close to the leading one. We use this quantity as a hallucination score, classifying an input text as hallucinated when its D-Score is larger than a pre-defined quantity. The motivation is that, when a model processes a text that conflicts with information available in its own internal state, the hidden representation may encode both the asserted content and some form of counter-evidence, uncertainty, correction, or lack of support; this can make the hidden trajectory spread across additional singular directions. We formalize this intuition through a lightweight spectral argument and evaluate the resulting detector on FAVA-Annotation and RAGTruth. The experiments indicate that the D-Score is a strong hidden-state signal for hallucination detection, while requiring no external verifier, no retrieval step, and no multiple generations.
\end{abstract}

\begin{keywords}
Large Language Models \sep Hallucination Detection \sep Hidden Activations \sep D-Score \sep Singular Values \sep Spectral Methods \sep Mechanistic Interpretability
\end{keywords}

\maketitle

\section{Introduction}

Large Language Models (LLMs) are now routinely used for question answering, summarization, dialogue, and open-ended generation, but they can still produce statements that are fluent and plausible while being false, contradicted, or not justified by the relevant evidence. This phenomenon is commonly referred to as \emph{hallucination}, and it remains one of the central obstacles to deploying LLMs in settings where reliability matters \cite{ji2023survey,farquhar2024semantic}. In this work, we use the term \emph{claim} to refer to a semantically checkable assertion expressed by a whole sentence or a portion of it. We define a claim to be \emph{supported} by an LLM when it can be proved correct by context or by prior information possessed by the model. Conversely, a claim is \emph{unsupported} when it is either recognized as false or uncertain by the model's internal knowledge. Under this terminology, hallucination detection is treated as the problem of identifying texts that contain at least one false or unsupported claim.

Existing approaches to hallucination detection often rely on information outside the original generation. Retrieval-based systems compare the generated text against external documents, while consistency-based methods sample several alternative answers from the same model and measure whether they agree with one another \cite{manakul2023selfcheckgpt,farquhar2024semantic}. These methods are useful, but they usually require additional inference passes, external resources, or a second model. This motivates a complementary question: \textit{can hallucination be detected directly from the internal states produced during a single forward pass?}

Several recent works suggest that the answer is positive. Internal-state methods such as INSIDE \cite{chen2024inside} and LLM-Check \cite{sriramanan2024llmcheck} show that hidden representations contain useful signals for detecting the truthfulness of generated statements. These results suggest that a model's internal geometry can contain information that is not fully visible from the generated text alone. In this paper, we follow this direction from a deliberately simple point of view. Rather than training an additional classifier on hidden states or comparing several sampled generations, we ask whether hallucination is reflected in the topology of the hidden-state trajectory produced by a single pass through the model.

Our starting intuition is that a well-supported text should induce a relatively coherent hidden trajectory, whose variation is concentrated along a small number of dominant directions. By contrast, when the text contains an unsupported claim that conflicts with what the model internally represents, the hidden states may have to encode the asserted content together with weaker signals of uncertainty. This can make the trajectory less concentrated and spread its variation across more directions, which can be detected by simple linear algebra tools. We therefore look for a simple way to measure how many directions carry a non-negligible part of the hidden-state energy.

This leads to the \emph{D-Score}. Given a fixed model, a fixed layer, and an input text, we consider the sequence of hidden states produced while processing the tokens of the text. The D-Score counts how many singular directions of this hidden-state trajectory remain close to the leading one, namely how many singular values $\sigma_i$ satisfy $\frac{\sigma_1}{\sigma_i} > \tau$ for a fixed $\tau > 0$. Since the ratio $\frac{\sigma_1}{\sigma_i}$ is related to the numerical dimensionality of the hidden states, a large D-Score signifies that the hidden trajectory spreads across many directions of comparable magnitude, while a small D-Score means that most of its variation is concentrated in a few dominant directions. Notably, computing the score does not require the full singular spectrum, which would be computationally heavy for recent, high-dimensional models: once the tolerance parameter $\tau$ and the decision level are fixed, it is sufficient to recover the leading singular values only up to the point where they fall below the relative floor using efficient algorithms such as the power method.

The detector induced by this idea is deliberately simple. After choosing the tolerance parameter $\tau$ and a calibrated decision level $\Dbar \in \mathbb{N}$, we predict hallucination whenever $\Dscore_\tau(\xt)>\Dbar$, where $\Dscore_\tau(\xt)$ indicates the D-score obtained by the considered model when fed with input text $\xt$. The score uses one forward pass, one hidden layer, and no external reference, and can be computed exactly from the singular values of the hidden activation matrix.

The contributions of this paper are as follows:
\begin{enumerate}
\item we introduce the D-Score $\Dscore_\tau(\xt)$ as a relative numerical rank of the hidden activation matrix of a fixed model and layer;
\item we give an operational definition of model-internal hallucination, based on unsupported claims to which the model assigns low internal truth belief;
\item we provide a geometric argument showing why hallucinated texts can increase the D-Score by spreading the hidden trajectory across additional conflict directions;
\item we evaluate D-Score against state-of-the-art methods, obtaining relatively strong results across \modelname{Llama-2-7B}, \modelname{Llama-3-8B}, and \modelname{Vicuna-7B} on multiple benchmarking datasets.
\end{enumerate}

The paper is organized as follows. In Section \ref{sec:related_works} we indicate the state-of-the-art related to our topic, highlighting the main difference from our approach. In Section \ref{sec:methodology} we introduce our D-score method, specifying the main mathematical properties and some details on the implementation. In Section \ref{sec:experiments} we report the experiments comparing the D-score against other common techniques for hallucination detection, showing that our method outperforms all of them in four different metrics, two datasets, and three open-weight models. Finally, in Section \ref{sec:discussion} we discuss our results and the limitations of our approach, and in Section \ref{sec:conclusion} we draw the conclusions. 

\section{Related Work}\label{sec:related_works}

Hallucination detection is a broad research area encompassing a wide range of approaches. Our work is most closely related to methods that leverage model uncertainty and internal representations to detect hallucinations, rather than relying solely on external retrieval. Hallucination has been studied across summarization, dialogue, data-to-text generation, question answering, and more recent LLM settings \cite{ji2023survey}. In this work, we focus on factual hallucinations, namely claims that are false, contradicted, invented, or not supported by the relevant evidence. This is the type of annotation provided by benchmarks such as FAVA-Annotation \cite{mishra2024fava}, which introduces a fine-grained taxonomy of hallucination types, and RAGTruth \cite{niu2024ragtruth}, which studies hallucinations in retrieval-augmented generation with manual span-level annotations.

A common way to detect hallucinations is to ask the model several times and measure whether the sampled answers are mutually consistent. SelfCheckGPT \cite{manakul2023selfcheckgpt} follows this idea in a black-box setting, using the intuition that facts known by the model should be repeated consistently across samples, while hallucinated facts should vary more. Semantic entropy methods refine this idea by grouping sampled outputs according to meaning, rather than by surface form, and then measuring uncertainty over semantic alternatives \cite{farquhar2024semantic}. These methods are powerful, but they require multiple generations and therefore increase inference cost.

A second line of work uses hidden activations directly. Azaria and Mitchell \cite{azaria2023internal} show that hidden states can be used to classify statements as true or false. Agrawal et al. \cite{agrawal2024references} study hallucinated references and show that language models often behave inconsistently around fabricated citations, suggesting that some information about the error is available internally. INSIDE \cite{chen2024inside} uses eigenvalue-based summaries of internal states to detect hallucination through semantic consistency in embedding space, while LLM-Check \cite{sriramanan2024llmcheck} studies hallucination detection from hidden states, attention maps, and output probabilities in a single-response setting. More recent work \cite{orgad2025know} argues that truthful information is present in internal representations, but also that it can be dataset-dependent and not always aligned with the final output. Our method belongs to this internal-state family, but it focuses on a simpler spectral statistic: the number of singular directions that remain large relative to the top one.

The D-Score is also related to effective rank and numerical rank ideas in linear algebra and signal processing \cite{roy2007effective}. The purpose, however, is different from estimating the intrinsic dimension of a dataset in the usual statistical sense. Here, the relative rank of a single hidden activation matrix is used as a diagnostic signal for the model's processing of a single text.

\section{Methodology}
\label{sec:methodology}

The proposed method is based on the singular spectrum of the hidden states produced by a fixed model and layer. We first define the activation matrix and the corresponding D-Score, then introduce the decision rule used by the detector, and finally give the model-internal interpretation that motivates the score. Throughout the section, $\bbone\{\mathcal A\}$ denotes the indicator of the event $\mathcal A$, namely $\bbone\{\mathcal A\}=1$ if $\mathcal A$ is true and $\bbone\{\mathcal A\}=0$ otherwise.

\subsection{Activation spectrum and D-Score}

Let $m\in\mcal$ be a fixed LLM with hidden dimension $d$, and let $j$ be a fixed layer. An input text is a token sequence $\xt=(x_1,\ldots,x_T)$, where $T$ denotes the length of the text in tokens.
Let $\htok^{(j)}_t(\xt;m)\in\R^d$ be the hidden state at layer $j$ associated with token $x_t$. We define the hidden activation matrix as
\begin{align*}
H^{(j)}(\xt;m)=\begin{bmatrix}\htok^{(j)}_1(\xt;m) & \cdots & \htok^{(j)}_T(\xt;m)\end{bmatrix}\in\R^{d\times T}.
\end{align*}
When $m$ and $j$ are fixed, we write $H(\xt)$ for simplicity. Let $H(\xt)=U\Sigma V^\top$ be the Singular Value Decomposition (SVD) of $H(\xt)$, and let
\begin{align*}
\sigma_1(H)\geq\sigma_2(H)\geq\cdots\geq\sigma_r(H)>0
\end{align*}
be its non-zero singular values, where $r:=\rank(H)$. We define the D-Score as a relative numerical rank of $H(\xt)$, namely the number of singular directions that remain large compared with the leading direction.

\begin{definition}[D-Score]
Let $H(\xt)=H^{(j)}(\xt;m)\neq 0$ and let $\tau>1$. We define the D-Score of $\xt$ for model $m$, layer $j$, and threshold $\tau$ as
\begin{align}
\begin{split}
\Dscore_\tau^{(j)}(\xt;m)
&=\sum_{i=1}^{r}\bbone\left\{\frac{\sigma_1(H(\xt))}{\sigma_i(H(\xt))}\leq\tau\right\} \\
&=\max\left\{n\leq r:\frac{\sigma_1(H(\xt))}{\sigma_n(H(\xt))}\leq\tau\right\}.
\label{eq:d_score}
\end{split}
\end{align}
When $m$ and $j$ are fixed, we simply write $\Dscore_\tau(\xt):=\Dscore_\tau^{(j)}(\xt;m)$.
\end{definition}

A small value of $\Dscore_\tau(\xt)$ indicates that most of the variation in the hidden trajectory is concentrated in a few dominant directions, while a larger value indicates that several directions remain active at comparable scale. The score is relative: it depends on ratios between singular values, and therefore it is not affected by a global rescaling of the activation matrix. This is important because the quantity is meant to capture the topology of the hidden-state trajectory, rather than the overall magnitude of the activations.

Given a tolerance parameter $\tau$ and a calibrated decision level $\Dbar\in\mathbb{N}$, the D-Score induces the binary detector
\begin{align}
\widehat{Y}_{\tau,\Dbar}^{(j)}(\xt;m)=\begin{cases}
    1 &\mbox{ if } \Dscore_\tau^{(j)}(\xt;m)>\Dbar, \\
    0 &\mbox{ otherwise.}
\end{cases}
\label{eq:classifier}
\end{align}
Thus, the text is classified as hallucinated when its hidden activation matrix has more than $\Dbar$ singular directions above the relative threshold determined by $\tau$. The parameters $\tau$, $j$, and $\Dbar$ specify the operating point of the detector and are selected on validation data. Once these quantities are fixed, the method requires one forward pass through the model and one spectral computation on the resulting hidden activation matrix.

In this paper, we interpret a high D-Score as evidence of increased hallucination risk. The central hypothesis is that, when a generated text contains an unsupported claim that is visible as such in the model's own internal state, the hidden representation may activate additional directions associated with conflict, uncertainty, correction, or lack of support. These directions can make the singular spectrum less concentrated, and the D-Score is designed to measure this form of spectral spreading.

\subsection{Model-internal hallucination and geometric rationale}

We now introduce a model-internal notion of hallucination to clarify the type of error for which the score is expected to be informative.

Let $\mathsf{Claims}(\xt)$ be the finite set of atomic factual claims asserted by $\xt$ in the relevant evaluation context, where a \emph{claim} is a checkable assertion expressed by the text, while \emph{atomic} means that the assertion is treated as a single unit for evaluation. A claim does not necessarily coincide with a full sentence. For example, the sentence
\begin{quote}
\textit{The Transformer architecture was introduced in 2017 and ChatGPT was released in 2022.}
\end{quote}
contains two factual claims, one about the introduction of the Transformer architecture and one about the release of ChatGPT. Conversely, a sentence such as
\begin{quote}
\textit{Transformers are more elegant than recurrent networks.}
\end{quote}
expresses a preference or qualitative judgment, and therefore is not a factual claim in the sense considered here.

The available evidence is the reference information with respect to which a claim is judged. In an open-domain benchmark, this may correspond to world knowledge or to the information used by human annotators to assess factuality. In a retrieval-augmented benchmark, it is usually the source document or the retrieved passages provided to the model. Thus, the same claim may be judged differently depending on the evaluation protocol: a claim can be true as a matter of world knowledge, but still unsupported in a retrieval-augmented setting if it does not follow from the provided passages.

Let $s(a)\in\{\texttt{true},\texttt{false},\bot\}$ be the support label of a claim $a\in\mathsf{Claims}(\xt)$. The value $s(a)=\texttt{true}$ means that $a$ is supported by the available evidence, $s(a)=\texttt{false}$ means that $a$ is contradicted by the available evidence, and $s(a)=\bot$ means that $a$ is not verifiable from the available evidence. Since the experiments evaluate hallucination as a binary condition, we introduce the collapsed label $s^+(a)=\bbone\{s(a)=\texttt{true}\}$. Therefore, $s^+(a)=0$ covers both contradictory and non-verifiable claims. We return to this binary reduction in the experimental section, where the dataset annotations are mapped to hallucinated and non-hallucinated examples.

Let $Z_m(\xt)$ denote the set of all internal information produced by model $m$ on $\xt$, for example hidden states, attention states, and logits. We define the model-internal truth belief of claim $a$ as the probability assigned by an ideal calibrated probe that has access to $Z_m(\xt)$:
\begin{align*}
b_m(a\mid\xt)=\mathbb{P}\bigl(s^+(a)=1\mid Z_m(\xt)\bigr).
\end{align*}
While this probe cannot be explicitly constructed, it is introduced here as a conceptual device: it separates unsupported claims that are, in principle, detectable from the model's internal state from unsupported claims for which the model may not contain the information needed to recognize the error.

\begin{definition}[Model-internal hallucination]
\label{def:model_internal_hallucination}
Fix a tolerance level $\beta\in(0,1/2)$. A text $\xt$ is said to be an internal hallucination for model $m$ at tolerance $\beta$ if $\mathsf{Claims}(\xt)$ contains an unsupported assertion to which the model assigns low internal truth belief:
\begin{align*}
Y_m^\beta(\xt)=\bbone\{\exists a\in\mathsf{Claims}(\xt): s^+(a)=0 \\\text{ and } b_m(a\mid\xt)\leq\beta\}.
\end{align*}
We write $Y_m^\beta(\xt)=1$ for hallucinated and $Y_m^\beta(\xt)=0$ for non-hallucinated.
\end{definition}

This definition separates two cases that are often conflated in hallucination detection. A claim can be unsupported according to the evaluation evidence, while the model may not internally represent enough information to recognize the problem. Conversely, a claim can be unsupported and also leave an internal trace, because the model represents information that makes the claim unlikely or inconsistent. The proposed method is intended for the second case. This is a limitation of the approach, but also the reason why a hidden-state detector can be meaningful without consulting external evidence at inference time.

The D-Score detector can now be interpreted as a simple proxy for this idealized target. It does not estimate $b_m(a\mid\xt)$ and does not identify the offending claim. Instead, it tests whether the hidden trajectory has an unusually large number of significant singular directions. The direction of the decision rule in Equation~\eqref{eq:classifier} reflects the assumption that internally visible unsupported content tends to make the hidden trajectory less concentrated. Grounded texts usually organize entity information, relations, syntactic structure, and continuation constraints around a coherent interpretation, so the corresponding activations can remain concentrated in a limited set of dominant directions. When the text asserts a claim that conflicts with information represented by the model, the hidden state may also carry traces of counter-evidence, uncertainty, correction, or missing support. These traces need not dominate the representation, but if they are sufficiently strong and not aligned with the directions used to encode the asserted content, they increase the number of singular values that remain large relative to $\sigma_1$.

This leads to the empirical hypothesis that we test:
\begin{hypothesis}[D-Score hallucination hypothesis]\label{hyp:dscore_large_values}
For a fixed model $m$, layer $j$, and spectral threshold $\tau$, large values of $\Dscore_\tau^{(j)}(\xt;m)$ are associated with a higher probability of model-internal hallucination. Equivalently, the mapping
\[
k\mapsto\mathbb{P}\left(Y_m^\beta(\xt)=1\mid\Dscore_\tau^{(j)}(\xt;m)=k\right)
\]
is monotonically non-decreasing for $k=1,\ldots,r$.
\end{hypothesis}

The hypothesis is intentionally empirical. It can fail when the model does not internally represent the error, when the error signal is too weak, or when the hidden-state change is mostly a change in magnitude rather than a change in direction. It can also fail when non-hallucinated texts are complex enough to activate several unrelated semantic directions. For this reason, the D-Score should be understood as a calibrated risk score, rather than as a universal characterization of hallucination.

The proposition below states a sufficient geometric mechanism behind the method. If unsupported content introduces $q$ sufficiently strong directions that are not already present in the coherent representation of the text, then the D-Score increases by at least $q$. Its proof is elementary and is given in Appendix~\ref{app:proofs}.

\begin{proposition}[Conflict directions increase D-Score]
\label{prop:conflict}
Let $S,C\in\R^{d\times T}$ be two hidden-state components. We interpret $S$ as the coherent component of the hidden trajectory, namely the part of the representation associated with processing the asserted content, the entities, the relations, and the local continuation constraints of the text. We interpret $C$ as an additional conflict component, namely the part of the representation associated with internally visible counter-evidence, uncertainty, correction, or missing support. Assume that the singular directions of $C$ are orthogonal to those of $S$, and that $\sigma_1(C)\leq\sigma_1(S)$. If there exists $q\in\mathbb{N}$ such that
\[
\sigma_1(C)\geq\cdots\geq\sigma_q(C)\geq\sigma_1(S)/\tau,
\]
then
\[
\Dscore_\tau(S+C)\geq\Dscore_\tau(S)+q.
\]
\end{proposition}

The orthogonality assumption is idealized, but it captures the mechanism behind the method. In practice, the added directions need not be exactly orthogonal, and the separation may only hold statistically, as discussed in Appendix~\ref{ssec:probabilistic_separation}. The proposition should therefore be read as a sufficient condition for an increase in D-Score, not as a complete model of how hallucinations are represented in transformer hidden states.

\subsection{Computation}
\label{ssec:dscore_computation}

This spectral characterization is useful because, in transformer models, the hidden dimension $d$ is typically much larger than the sequence length $T$. It is therefore convenient to work with the Gram matrix $G:=H^\top H\in\R^{T\times T}$. By Property~\ref{prop:dscore_gram_matrix} of Proposition~\ref{prop:computational_properties}, if $\lambda_i(G)=\sigma_i(H)^2$, computing $\Dscore_\tau(H)$ amounts to counting how many eigenvalues of $G$ remain above the relative threshold $\lambda_1(G)/\tau^2$. Thus, one does not need to manipulate the full $d\times T$ hidden-state matrix through a complete SVD in order to evaluate the score.

A full SVD of $H$ costs $\mathcal{O}(dT^2+T^3)$ when $d\geq T$, after the hidden states have been extracted; equivalently, one may form $G$ explicitly at cost $\mathcal{O}(dT^2)$ and compute its full eigendecomposition at cost $\mathcal{O}(T^3)$ \cite{golub2013matrix,trefethen1997numerical}. However, the D-Score only depends on the point where the spectrum falls below a relative threshold, so the full decomposition is not always necessary. When this crossing occurs early, partial eigensolvers such as power iteration, block power iteration, or Lanczos-style methods can estimate only the leading part of the spectrum \cite{trefethen1997numerical,halko2011finding}.

These methods typically require matrix-vector products of the form $v\mapsto Gv=H^\top(Hv)$, each of which costs $\mathcal{O}(dT)$ without explicitly forming $G$. If the procedure estimates $K$ eigenvalues and uses $N_{\mathrm{it}}$ iterations per eigenvalue, the leading cost is approximately $\mathcal{O}(N_{\mathrm{it}}K dT)$, up to reorthogonalization costs. This can be substantially cheaper than a full decomposition when $K\ll T$, which is precisely the regime in which the D-Score threshold is crossed after a small number of singular directions.

\section{Experimental Evaluation}
\label{sec:experiments}

We evaluate whether the D-Score provides a useful hidden-state signal for hallucination detection. The experiments are designed to test two related questions: whether the singular spectrum of hidden activations separates hallucinated from non-hallucinated outputs, and whether this signal remains useful across different model families and hallucination settings. We therefore consider both open-domain factual hallucinations and hallucinations relative to retrieved or supplied evidence, and we compare the D-Score against single-response uncertainty scores and the hidden-state baseline that is closest to our method.

\subsection{Experimental protocol}

We evaluate on two widely-used hallucination benchmarks: FAVA-Annotation \cite{mishra2024fava} and RAGTruth \cite{niu2024ragtruth}. FAVA-Annotation provides fine-grained hallucination annotations for language-model outputs, with error types including entity errors, relation errors, contradictory statements, invented statements, subjective statements, and unverifiable statements. RAGTruth is a hallucination corpus for retrieval-augmented generation, with manually annotated responses and span-level hallucination labels across RAG tasks. For both datasets, we convert the annotation to a binary response-level label indicating whether the text contains at least one hallucinated span. The two benchmarks are complementary: FAVA emphasizes factual errors in model outputs, while RAGTruth tests whether the generated response is supported by the evidence available in the retrieval-augmented setting.

We use three open-source instruction or chat models: \modelname{Llama-2-7B} \cite{touvron2023llama}, \modelname{Llama-3-8B} \cite{dubey2024llama}, and \modelname{Vicuna-7B} \cite{chiang2023vicuna}. These models are representative of the model families commonly used in prior hallucination-detection work. The models are loaded through HuggingFace Transformers \cite{wolf2020transformers}. For each input $\xt$, we run a single forward pass and extract the hidden states layer by layer. Unless otherwise stated, $H_m^{(j)}(\xt)$ is the uncentered matrix whose columns are token hidden states at layer $j$. Note that this is different to what is done in many competing works, where the model is asked to process multiple realizations of the same concept to classify the presence of hallucinations. In particular, our approach is closer to real applications where one wants to classify the presence of hallucination in a real-time conversation, where multiple realizations of the same prompt are typically unavailable.

For each model $m$ and layer $j$, we compute $\Dscore_\tau^{(j)}(\xt;m)$ using \eqref{eq:d_score} with singular values computed efficiently through the power method, as described in Section the \ref{ssec:dscore_computation}, for a user-defined value of $\tau > 0$. The performance of our method is measured through four commonly-used metrics, namely the AUROC, the accuracy, the TPR@5\%FPR, and the F1 Score. The AUROC is computed by using $\Dscore_\tau^{(j)}(\xt;m)$ directly as the scalar hallucination score, accuracy and F1 score are computed by selecting a threshold $\Dbar$ and applying Equation~\eqref{eq:classifier}, while TPR@5\%FPR is computed by sweeping $\Dbar$ and selecting the operating point whose false positive rate is closest to 5\%. The results below use the best-performing layer for each model, following the standard diagnostic convention in hidden-state analysis. In a deployment setting, the layer, $\tau$, and $\Dbar$ should instead be selected on a calibration set and then kept fixed. However, as reported in Appendix \ref{app:dscore_robustness}, we empirically proved that D-Score is particularly robust to differences in the choice of $\tau$, $\Dbar$ and layer.

We compared D-Score against Hidden Score from \cite{sriramanan2024llmcheck}, which is the closest prior hidden-state spectral score. Where available under the same protocol, we also report perplexity-based scoring, window entropy, and logit entropy. These baselines are useful because they do not require external retrieval or multiple generations, and therefore isolate the value of the hidden-state spectral information used by the D-Score.

\subsection{Results}

Table~\ref{tab:fava} reports results on FAVA-Annotation. Across all three models, the D-Score obtains the highest performance among all the listed methods. Interestingly, the difference in performance between D-Score and Hidden Score is more pronounced for more advanced models, reaching a difference of $9.88$ in AUROC for \texttt{Llama-3-8B}, reflecting the observation that D-Score can only better identify those hallucinations that are recognized by the model as non-factual, a property that correlates with its capacity.

\begin{table*}[tbhp]
\caption{Hallucination detection results on FAVA-Annotation. All values are reported in percentage points, and higher is better for every metric. D-Score uses $\Dscore_\tau(\xt)$ as the scalar score for AUROC and the threshold rule in Eq.~\eqref{eq:classifier} for thresholded metrics.}
\label{tab:fava}
\centering
\begin{tabular}{llcccc}
\toprule
Model & Method & AUROC & Acc. & TPR@5\%FPR & F1 Score \\
\midrule
\multirow{5}{*}{Llama-2-7B}
  & PPL Score & 53.22 & 58.68 & 3.59 & 68.33 \\
  & Window Entropy & 56.90 & 56.59 & 2.99 & 42.52 \\
  & Logit Entropy & 53.80 & 55.99 & 2.99 & 56.73 \\
  & Hidden Score & 58.44 & 58.08 & 11.98 & 59.66 \\
  & \textbf{D-Score} & \textbf{62.87} & \textbf{62.87} & \textbf{13.17} & \textbf{68.95} \\
\midrule
\multirow{5}{*}{Llama-3-8B}
  & PPL Score & 53.22 & 58.68 & 3.59 & 67.40 \\
  & Window Entropy & 56.90 & 56.59 & 2.99 & 55.52 \\
  & Logit Entropy & 53.80 & 55.99 & 2.99 & 56.27 \\
  & Hidden Score & 57.10 & 57.78 & 10.78 & 65.38 \\
  & \textbf{D-Score} & \textbf{66.98} & \textbf{65.57} & \textbf{17.96} & \textbf{68.64} \\
\midrule
\multirow{5}{*}{Vicuna-7B}
  & PPL Score & 53.96 & 56.89 & 3.59 & 64.20 \\
  & Window Entropy & 55.24 & 58.38 & 5.99 & 66.02 \\
  & Logit Entropy & 52.29 & 55.69 & 1.80 & 57.31 \\
  & Hidden Score & 58.22 & 59.28 & 10.18 & 66.99 \\
  & \textbf{D-Score} & \textbf{63.93} & \textbf{63.17} & \textbf{13.77} & \textbf{69.67} \\
\bottomrule
\end{tabular}
\end{table*}

Similar results are observed on the RAGTruth benchmark, reported in Table~\ref{tab:ragtruth}, where we only included Hidden Score as a competing method on \texttt{Llama-3-8B} and \texttt{Vicuna-7B} since the other methods did not report the results on this benchmark in their original paper. The D-Score again improves all the metrics for all three models, with gains that are more pronounced on more capable models, reaching an AUROC of $61.53$ for \texttt{Llama-3-8B}: a much larger value than the $54.95$ obtained by Hidden Score. Interestingly, while the AUROC of Hidden Score largely dropped on \texttt{Vicuna-7B}, reaching a value comparable to pure randomness, D-Score keeps its performance, reaching an AUROC value of $60.20$. Overall, the RAGTruth results indicate that the D-Score is not only sensitive to open-domain factual errors, but also to support errors defined with respect to supplied evidence.

\begin{table*}[tbhp]
\caption{Hallucination detection results on the RAGTruth summarization subset. All values are reported in percentage points, and higher is better for every metric.}
\label{tab:ragtruth}
\centering
\begin{tabular}{llcccc}
\toprule
Model & Method & AUROC & Acc. & TPR@5\%FPR & F1 Score \\
\midrule
\multirow{5}{*}{Llama-2-7B}
  & PPL Score & 53.73 & 54.07 & 7.69 & 58.70 \\
  & Window Entropy & 52.08 & 53.17 & 4.98 & 53.98 \\
  & Logit Entropy & 53.95 & 55.43 & 7.24 & 53.74 \\
  & Hidden Score & 54.11 & 56.33 & 8.14 & 61.51 \\
  & \textbf{D-Score} & \textbf{57.30} & \textbf{57.47} & \textbf{12.03} & \textbf{67.54} \\
\midrule
\multirow{2}{*}{Llama-3-8B}
  & Hidden Score & 54.95 & 56.02 & \textbf{15.77} & 30.97 \\
  & \textbf{D-Score} & \textbf{61.53} & \textbf{60.79} & 12.86 & \textbf{67.60} \\
\midrule
\multirow{2}{*}{Vicuna-7B}
  & Hidden Score & 51.93 & 56.02 & 12.45 & 59.89 \\
  & \textbf{D-Score} & \textbf{60.20} & \textbf{59.75} & \textbf{13.69} & \textbf{67.86} \\
\bottomrule
\end{tabular}
\end{table*}

Taken together, the two benchmarks show a consistent pattern: the D-Score improves over Hidden Score for every model and dataset pair, and the improvements are often larger than the differences among the simpler uncertainty baselines. This is the main empirical evidence for the central claim of the paper, namely that the relative number of active singular directions in the hidden activation matrix carries information about hallucination risk.

\subsection{The effect of model capacity}
\begin{table}[tbhp]
\caption{Hallucination detection results on the FAVA subset without hallucinated sample that are \textbf{not explicitly recognized} as hallucinated by the model. All values are reported in percentage points, and higher is better for every metric.}
\label{tab:fava_subset_hall}
\centering
\begin{tabular}{llc}
\toprule
Model & Method & AUROC \\
\midrule
\multirow{2}{*}{Llama-2-7B (n samp = 256)}
  & Hidden Score & 56.26 \\
  & \textbf{D-Score} & \textbf{62.79} \\
\midrule
\multirow{2}{*}{Llama-3-8B (n samp = 118)}
  & Hidden Score & 62.17 \\
  & \textbf{D-Score} & \textbf{72.51} \\
\midrule
\multirow{2}{*}{Vicuna-7B (n samp = 134)}
  & Hidden Score & 59.56 \\
  & \textbf{D-Score} & \textbf{65.46} \\
\bottomrule
\end{tabular}
\end{table}

\begin{table}[tbhp]
\caption{Hallucination detection results on the FAVA subset without non hallucinated sample that are \textbf{explicitly recognized} as hallucinated by the model. All values are reported in percentage points, and higher is better for every metric.}
\label{tab:fava_subset_nonhall}
\centering
\begin{tabular}{llc}
\toprule
Model & Method & AUROC \\
\midrule
\multirow{2}{*}{Llama-2-7B (n samp = 248)}
  & Hidden Score & 56.34 \\
  & \textbf{D-Score} & \textbf{61.86} \\
\midrule
\multirow{2}{*}{Llama-3-8B (n samp = 326)}
  & Hidden Score & 56.27 \\
  & \textbf{D-Score} & \textbf{67.14} \\
\midrule
\multirow{2}{*}{Vicuna-7B (n samp = 318)}
  & Hidden Score & 57.51 \\
  & \textbf{D-Score} & \textbf{63.16} \\
\bottomrule
\end{tabular}
\end{table}
To better understand the behaviour of the D-Score, we conduct a subset analysis on the FAVA-Annotation dataset that isolates the cases in which the model's explicit self-assessment agrees with the human annotation. To obtain the model's explicit hallucination judgement, we prompt each model directly on the same prompt-response pairs used for D-Score evaluation. Specifically, the model receives the original prompt and response concatenated into a structured instruction, and is asked to return a binary answer of the form \texttt{hallucination: true/false}. Generation uses a temperature of zero, making the output fully deterministic, and the response is parsed by checking for the presence of the strings \texttt{true} or \texttt{false} in the decoded output. Samples for which the model produces neither string are discarded. This procedure yields a model-level binary label that is independent of the D-Score and of the human annotation.
The two subsets are then constructed as follows. The first subset retains only hallucinated samples that the model explicitly recognizes as such, removing the cases in which the model fails to flag its own errors. The second subset retains only non-hallucinated samples that the model correctly does not flag, removing the cases in which the model raises a false alarm. Both subsets therefore represent a cleaner evaluation setting in which the model's self-assessment is aligned with the ground truth.
These two subsets serve a precise diagnostic purpose. By removing the samples on which the model's explicit judgement disagrees with the annotation, one obtains a more reliable evaluation set in which the hallucination signal is expected to be stronger, as it eliminates the scenario where D-Score does not identify the hallucination from the hidden state of the model because that input is not a model-internal hallucination as in Definition \ref{def:model_internal_hallucination}, and therefore does not satisfies the hypothesis of Proposition \ref{prop:conflict}. The results for the two subsets are reported in Table~\ref{tab:fava_subset_nonhall} and Table~\ref{tab:fava_subset_hall}. For both the subsets, the D-Score consistently outperforms Hidden Score across all three models. Interestingly, D-Score also largely improves its performance compared to the experiments on the full dataset, reaching an impressive 72.51 of AUROC on \texttt{Llama-3-8B}, proving that a big portion of the misclassification error from the experiment on the full datasets is due to the model being unable to recognize that the provided claim is false.

\section{Discussion, Limitations, and Future Work}
\label{sec:discussion}

The empirical results support the view that hallucination detection can be approached as a hidden-state geometry problem. Rather than using the average magnitude of the spectrum, the probability of the generated sequence, or token-level entropy alone, the D-Score asks how many singular directions remain important before the spectrum falls below a relative threshold. This makes it a natural statistic for the setting considered in this paper: if an internally visible hallucination introduces a conflict between the asserted content and information represented by the model, then the hidden trajectory may become less concentrated in a small number of dominant directions.

At the same time, the D-Score should be interpreted as a risk signal rather than as a direct test of truth. A high value of $\Dscore_\tau(\xt)$ indicates that the hidden representation is spread across many comparable directions, which may happen when the model is processing unsupported or conflicting content. However, the same spectral behavior could also be caused by long, diverse, or semantically complex texts. Conversely, a hallucination may go undetected if the model does not internally represent the relevant evidence, if the conflict signal is too weak, or if the hidden-state change appears mainly as a change in magnitude rather than as the activation of additional directions.

There are also practical limitations. The method requires white-box access to hidden activations, so it cannot be directly applied to closed models unless such activations are exposed. It depends on the selected layer, the spectral threshold $\tau$, and the operating threshold $\Dbar$. Since $\Dscore_\tau(\xt)\leq T$, sequence length can influence the range of the score, and length-stratified calibration may be important in deployment. The use of the best-performing layer is appropriate for diagnostic analysis, but a deployed detector should select the layer and thresholds on a held-out calibration set and keep them fixed at test time.

Future work should complete the comparison against a broader family of baselines, including semantic entropy approximations, consistency-based detectors, probing classifiers, attention-based scores, and retrieval-based factuality checkers. It would also be useful to report layer-wise curves, threshold-transfer experiments across datasets, and length-stratified results, since these analyses would clarify whether the D-Score captures a stable hallucination signal or partly reflects dataset-specific differences in text length and complexity. On the methodological side, natural extensions include centered or normalized hidden-state matrices, token-local variants of the score, and span-level versions that could identify where in the response the hallucination signal appears. On the theory side, the conflict-direction argument should be refined into a more realistic model of transformer hidden states, where the relevant subspaces are not exactly orthogonal and the conflict signal may appear only at specific token positions.

\section{Conclusion}
\label{sec:conclusion}

We introduced the D-Score $\Dscore_\tau(\xt)$ of an input text as a relative numerical rank of the hidden activation matrix of a fixed model and layer. The resulting detector is simple: it classifies a text as hallucinated when $\Dscore_\tau(\xt)$ exceeds a calibrated threshold $\Dbar$. The method is motivated by the idea that internally visible hallucinations may activate additional conflict directions in hidden-state space, increasing the number of singular values that remain large relative to the leading one. Experiments on FAVA-Annotation and RAGTruth show that this single-pass spectral count is a strong hallucination signal across multiple open-source models, while requiring no external verifier, no retrieval step at detection time, and no multiple generations.

\section*{Acknowledgements and Declarations}
This research did not receive any specific grant from funding agencies in the public, commercial, or not-for-profit sectors.
The authors declare no competing interests.

\printcredits

\bibliographystyle{cas-model2-names}
\bibliography{bibliography}

\appendix

\section{Additional Mathematical Details}
\label{app:proofs}

This appendix provides the mathematical details that support the definitions and arguments introduced in the main text. The results clarify the main properties of the D-Score and the conditions under which the proposed geometric interpretation holds.

\subsection{Proof of Proposition~\ref{prop:conflict}}

\begin{proof}
Under the orthogonality assumption of Proposition~\ref{prop:conflict}, the row and column spaces associated with $S$ and $C$ can be represented in mutually orthogonal bases. Therefore, there exist orthogonal matrices $P$ and $Q$ such that
\begin{align*}
P^\top(S+C)Q=
\begin{bmatrix}
A & 0\\
0 & B
\end{bmatrix},
\end{align*}
where the singular values of $A$ are those of $S$, while the singular values of $B$ are those of $C$. It follows that the singular values of $S+C$ are the union of the singular values of $S$ and $C$.

Since $\sigma_1(C)\leq\sigma_1(S)$, the largest singular value of $S+C$ is $\sigma_1(S)$. Hence, the relative threshold used to compute $\Dscore_\tau(S+C)$ is the same as the threshold used for $\Dscore_\tau(S)$, namely $\sigma_1(S)/\tau$. Every singular value of $S$ counted by $\Dscore_\tau(S)$ is therefore still counted in $S+C$. Moreover, by assumption, the first $q$ singular values of $C$ satisfy
\begin{align*}
\sigma_i(C)\geq\frac{\sigma_1(S)}{\tau},
\qquad i=1,\ldots,q,
\end{align*}
and are counted as additional directions. Consequently,
\begin{align*}
\Dscore_\tau(S+C)\geq\Dscore_\tau(S)+q.
\end{align*}
\end{proof}

\subsection{Properties of \texorpdfstring{$\Dscore_\tau(H)$}{the D-Score}}

\begin{proposition}[Properties of $\Dscore_\tau(H)$]
\label{prop:computational_properties}
Let $H\in\R^{d\times T}$ be nonzero and let $\tau>1$.
\begin{enumerate}[label=(\roman*)]
\item \label{prop:dscore_scale_invariance}
If $c\neq 0$, then $\Dscore_\tau(cH)=\Dscore_\tau(H)$.

\item \label{prop:dscore_orthogonal_invariance}
If $U$ and $V$ are orthogonal matrices of compatible dimensions, then $\Dscore_\tau(UHV)=\Dscore_\tau(H)$.

\item \label{prop:dscore_gram_matrix}
If $\lambda_i$ are the nonzero eigenvalues of $H^\top H$, sorted in nonincreasing order, then
\begin{align*}
\Dscore_\tau(H)
=
\sum_i
\bbone\left\{
\lambda_i\geq\frac{\lambda_1}{\tau^2}
\right\}.
\end{align*}

\item \label{prop:dscore_stable_rank_bound}
If $\operatorname{srank}(H)=\norm{H}_F^2/\norm{H}_2^2$ is the stable rank of $H$, then
\begin{align*}
\Dscore_\tau(H)\leq\tau^2\operatorname{srank}(H).
\end{align*}
\end{enumerate}
\end{proposition}

The first two properties show that the D-Score depends on the shape of the singular spectrum rather than on the scale or coordinate system of the hidden states. In particular, it is unchanged by a global rescaling of the activation matrix and by orthogonal changes of basis. The third property provides the formulation used in computation: since the squared singular values of $H$ are the nonzero eigenvalues of $H^\top H$, the D-Score can be obtained from the smaller token-level Gram matrix. The final property connects the D-Score with stable rank, a standard measure of how broadly the energy of a matrix is distributed across its singular directions.

\begin{proof}
If $c\neq 0$, every singular value of $cH$ is equal to $|c|$ times the corresponding singular value of $H$. Hence,
\begin{align*}
\frac{\sigma_1(cH)}{\sigma_i(cH)}
=
\frac{|c|\sigma_1(H)}{|c|\sigma_i(H)}
=
\frac{\sigma_1(H)}{\sigma_i(H)}.
\end{align*}
The set of singular values satisfying the relative threshold is unchanged, and therefore $\Dscore_\tau(cH)=\Dscore_\tau(H)$.

If $U$ and $V$ are orthogonal matrices of compatible dimensions, then $UHV$ has the same singular values as $H$. It follows directly that $\Dscore_\tau(UHV)=\Dscore_\tau(H)$.

For the third property, the nonzero eigenvalues of $H^\top H$ satisfy $\lambda_i=\sigma_i(H)^2$. Therefore, $\sigma_i(H)\geq\sigma_1(H)/\tau$ if and only if $\lambda_i\geq\lambda_1/\tau^2$. Counting the eigenvalues that satisfy this condition gives
\begin{align*}
\Dscore_\tau(H)
=
\sum_i
\bbone\left\{
\lambda_i\geq\frac{\lambda_1}{\tau^2}
\right\}.
\end{align*}

Finally, let $k=\Dscore_\tau(H)$. By definition, $\sigma_i(H)\geq\sigma_1(H)/\tau$ for $i=1,\ldots,k$. It follows that
\begin{align*}
\norm{H}_F^2
=
\sum_i\sigma_i(H)^2
\geq
\sum_{i=1}^k\sigma_i(H)^2
\geq
k\frac{\sigma_1(H)^2}{\tau^2}.
\end{align*}
Since $\sigma_1(H)=\norm{H}_2$, division by $\norm{H}_2^2$ gives
\begin{align*}
\operatorname{srank}(H)
=
\frac{\norm{H}_F^2}{\norm{H}_2^2}
\geq
\frac{k}{\tau^2}.
\end{align*}
Rearranging the inequality yields $\Dscore_\tau(H)\leq\tau^2\operatorname{srank}(H)$.
\end{proof}

\subsection{A simple probabilistic separation statement}
\label{ssec:probabilistic_separation}

The geometric argument in the main text can also be expressed as a statistical separation condition. Let $Y\in\{0,1\}$ denote the hallucination label and let $D=\Dscore_\tau(\xt)$. Suppose that there exist two integers $a<b$ such that
\begin{align*}
\mathbb{P}(D\leq a\mid Y=0)&\geq 1-\delta_0,\\
\mathbb{P}(D\geq b\mid Y=1)&\geq 1-\delta_1.
\end{align*}
The first condition states that non-hallucinated examples have a D-Score no larger than $a$ with high probability, while the second states that hallucinated examples have a D-Score of at least $b$ with high probability.

Consider any decision threshold $\Dbar$ satisfying $a\leq\Dbar<b$. A classification error can occur only when a non-hallucinated example has $D>a$, or when a hallucinated example has $D<b$. Therefore,
\begin{align*}
\mathbb{P}\left(\bbone\{D>\Dbar\}\neq Y\right)
&\leq
\mathbb{P}(Y=0)\delta_0+\mathbb{P}(Y=1)\delta_1\\
&\leq
\delta_0+\delta_1.
\end{align*}
This result formalizes a simple point: when the D-Score distributions of the two classes are separated with high probability, every threshold placed between the two concentration regions has a correspondingly small classification error. The terms $\delta_0$ and $\delta_1$ account for the overlapping tails of the two distributions.

\section{Algorithmic Details}
\label{app:algorithmic_details}

The D-Score can be computed exactly through an SVD of $H$ or an eigendecomposition of the Gram matrix $G=H^\top H$. Since the score only counts eigenvalues above the relative threshold $\theta=\lambda_1/\tau^2$, a complete decomposition may perform more work than needed. A partial eigensolver can instead estimate the eigenvalues in decreasing order and stop as soon as the next estimated eigenvalue falls below $\theta$.

Algorithm~\ref{alg:power} describes this procedure using deflated power iteration. The algorithm does not need to form $G$ explicitly, since each multiplication by $G$ can be evaluated as $g(v)=H^\top(Hv)$. The first power iteration estimates the largest eigenvalue $\lambda_1$ and defines the threshold $\theta$. The following iterations estimate the remaining eigenvalues while keeping the current vector orthogonal to the eigenvectors that have already been found.

\begin{algorithm}[tbhp]
\caption{D-Score by deflated power iteration}
\label{alg:power}
\begin{algorithmic}[1]
\Require $H\in\R^{d\times T}$, $\tau>1$, $K_{\max}$, $N_{\mathrm{it}}$
\Ensure $\widehat{\Dscore}_\tau(H)$
\State Define $g(v)=H^\top(Hv)$
\State Estimate the leading eigenpair $(\widehat{\lambda}_1,q_1)$ of $H^\top H$ by power iteration
\State Set $\theta=\widehat{\lambda}_1/\tau^2$, $Q=[q_1]$
\For{$k=2,\ldots,K_{\max}$}
  \State Draw a random unit vector $v$ orthogonal to $Q$
  \For{$t=1,\ldots,N_{\mathrm{it}}$}
    \State $v\leftarrow g(v)-Q(Q^\top g(v))$
    \State $v\leftarrow v/\norm{v}_2$
  \EndFor
  \State $\widehat{\lambda}_k\leftarrow v^\top g(v)$
  \If{$\widehat{\lambda}_k<\theta$}
    \State \Return $k-1$
  \EndIf
  \State $Q\leftarrow[Q\;\;v]$
\EndFor
\State \Return $K_{\max}$
\end{algorithmic}
\end{algorithm}

The returned value is approximate because the eigenvalues are estimated from a finite number of iterations. In particular, an inaccurate estimate of $\lambda_1$ also changes the threshold $\theta$. The number of iterations should therefore be large enough to estimate the leading eigenvalues accurately around the threshold crossing. Instead of always using a fixed number of iterations, an implementation may stop when the change in the Rayleigh quotient falls below a chosen tolerance.

The parameter $K_{\max}$ limits the number of eigenvalues that the algorithm attempts to recover. If the threshold crossing is reached before $K_{\max}$, the algorithm returns the number of estimated eigenvalues above $\theta$. If all estimated eigenvalues remain above the threshold, the returned value is only a lower bound unless $K_{\max}\geq\rank(H)$. In this setting, $K_{\max}$ can be set to at most $\min\{d,T\}$, although a smaller value is sufficient when the spectrum usually crosses the threshold early.

If $K$ eigenvalues are evaluated using $N_{\mathrm{it}}$ matrix-vector iterations, the products involving $H$ require approximately $\mathcal{O}(N_{\mathrm{it}}KdT)$ operations. Orthogonalization adds a cost that grows with $K$, but remains limited when the threshold is crossed after a small number of directions. The partial computation can therefore be cheaper than a complete eigendecomposition when $K\ll T$.

\section{Robustness of the D-Score}
\label{app:dscore_robustness}

\begin{figure*}[bthp]
\centering
\includegraphics[width=\linewidth]{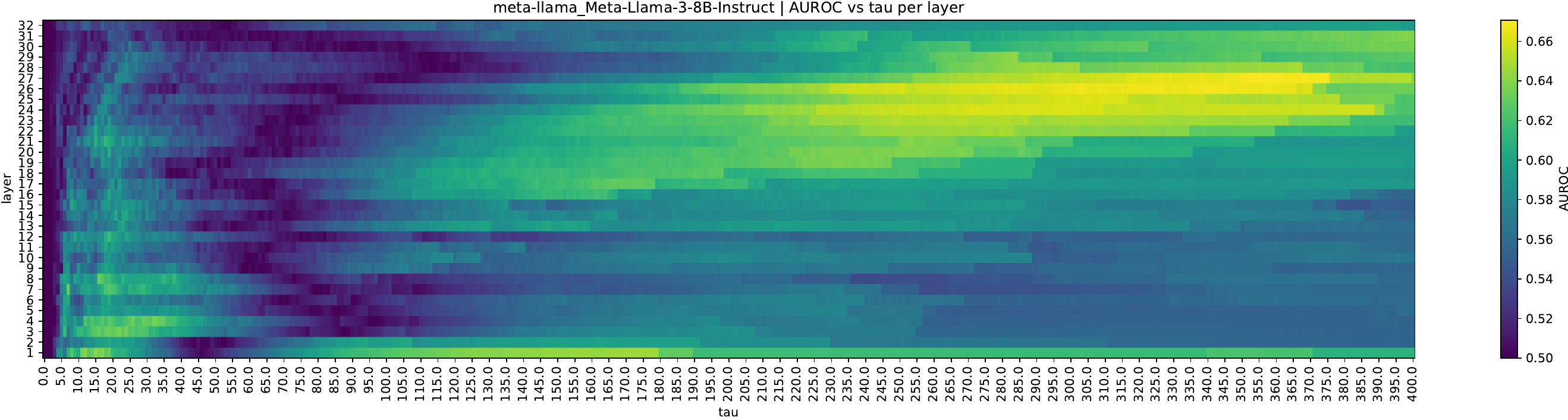}
\caption{AUROC of the D-Score detector on FAVA-Annotation with Llama-3-8B-Instruct as a function of the tolerance parameter $\tau$ on the horizontal axis and the layer $j$ on the vertical axis. A broad region of nearby configurations achieves similar performance, indicating that the score is not tied to one isolated choice of $\tau$ and $j$.}
\label{fig:auroc_heatmap}
\end{figure*}

The D-Score detector depends on the tolerance parameter $\tau$, the selected layer $j$, and the decision threshold $\Dbar$. The threshold $\Dbar$ controls the operating point after the scalar score has been computed, while $\tau$ and $j$ determine the score itself. We examine the sensitivity to $\tau$ and $j$ using FAVA-Annotation and Llama-3-8B-Instruct as a representative setting. Since AUROC evaluates the ranking induced by the scalar D-Score across all possible values of $\Dbar$, this analysis does not require fixing a particular decision threshold.

Figure~\ref{fig:auroc_heatmap} reports the AUROC obtained from $D=\Dscore_\tau^{(j)}(\xt;m)$ for different values of $\tau$ and different layers $j$. The heatmap shows a broad region in which several nearby layer--$\tau$ combinations achieve similar AUROC values. Performance changes gradually across much of this region, rather than dropping sharply when either parameter is moved away from the best configuration.

This result suggests that the observed performance is not tied to one isolated combination of layer and tolerance. The figure remains a diagnostic analysis on one dataset and one model, so $\tau$ and $j$ should still be selected on validation data and then kept fixed when evaluating new examples. The heatmap shows that this selection allows some flexibility around the best-performing region.

This supports validation-based parameter selection and reduces the concern that the reported performance is caused by a single unstable configuration. Further analyses could examine the stability of the calibrated decision threshold $\Dbar$, the transfer of the selected parameters across datasets, and the effect of sequence length on the score distribution.

\newpage

\bio{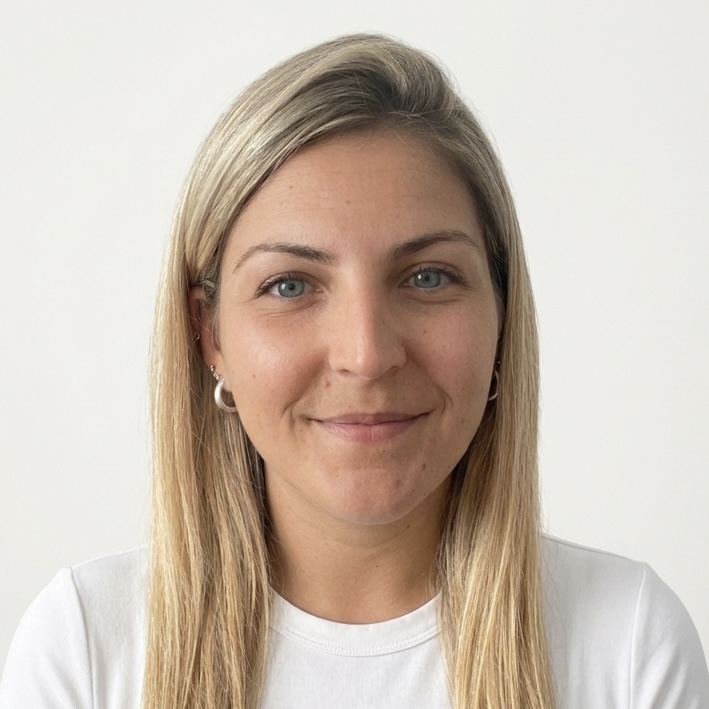}
Bianca Raimondi holds a Master's degree in Computer Science from the University of Bologna. She is currently a PhD student specialising in Data Science and Computation. Her research focuses on applying Large Language Models in education, particularly examining the biases of these models and how they represent information internally.
\endbio

\bio{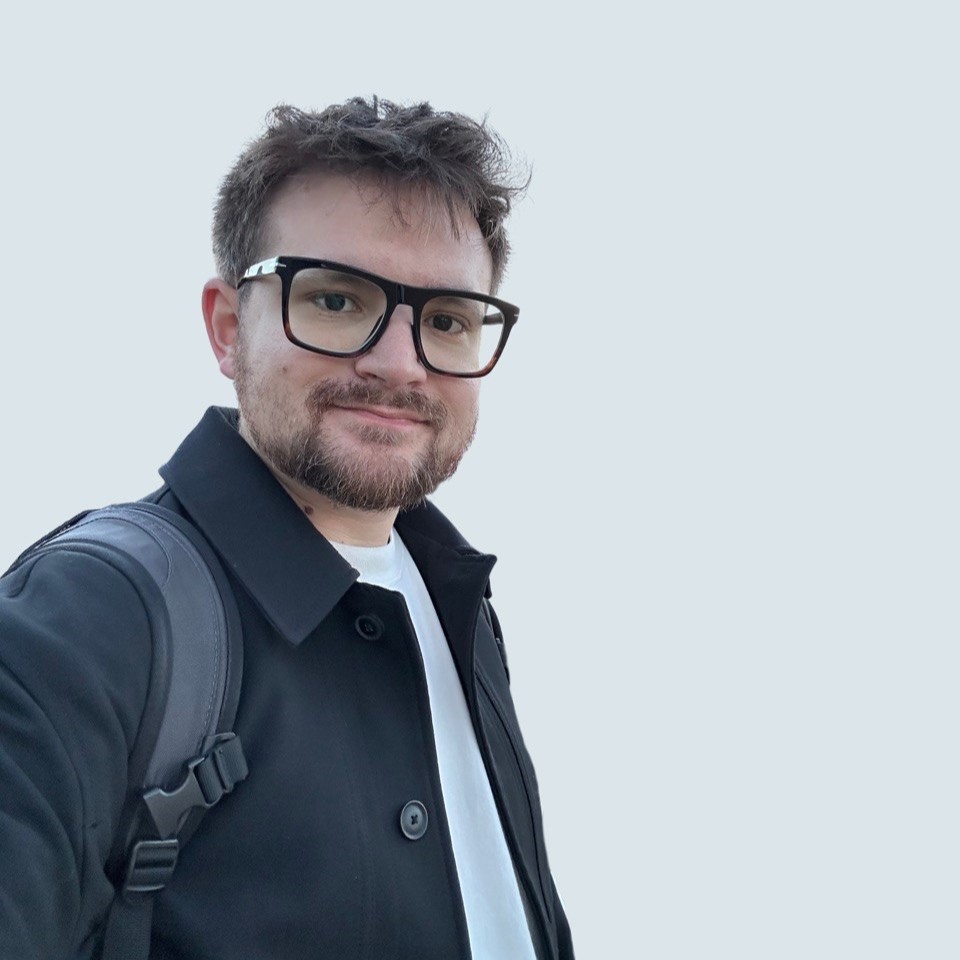}
Davide Evangelista is an Assistant Professor at the Department of Computer Science and Engineering at the University of Bologna. His research focuses on the use of Deep Generative Models for medical image reconstruction, with particular interest in improving the quality and reliability of imaging methodologies through data-driven approaches. More recently, his work has expanded toward the foundational properties of Deep Generative Models and Large Language Models, exploring their theoretical understanding and broader implications for machine learning systems.
\endbio

\bio{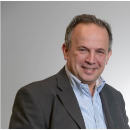}
Maurizio Gabbrielli is professor of Computer Science since 2001 at the Department of Computer Science and Engineering of the University of Bologna and Associate dean for AI at Bologna Business School. He has been Head of the Department of Computer Science and Engineering and member of the INRIA project team FOCUS. He received his Ph.D. in Computer Science in 1992 from the University of Pisa and has been employed at Centrum Wiskunde \& Informatica (CWI, Amsterdam), at the University of Pisa and at the University of Udine.
\endbio

\bio{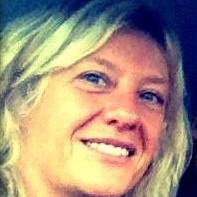}
Elena Loli Piccolomini is Full professor of Numerical Analysis at the Department of Computer Science and Engineering at the University of Bologna. Her research focuses on computational imaging, with particular emphasis on medical imaging and tomographic reconstruction using variational methods and deep learning techniques.
\endbio

\end{document}